\documentclass[onecolumn,10pt,a4paper]{article}

\usepackage{graphicx,amsfonts,psfrag,fancyhdr,layout,appendix}
\usepackage[ansinew]{inputenc}
\usepackage[english]{babel}
\usepackage{amsmath}
\usepackage{amsthm}
\usepackage{latexsym}
\usepackage{algorithm}
\usepackage{algorithmic}
\usepackage{natbib}

\newtheorem{defin}{Definition}
\newtheorem{thm}{Theorem}
\newtheorem{lem}{Lemma}
\newtheorem{cor}{Corollary}

\begin{document}

\title{Fixed-point and coordinate descent algorithms for regularized kernel methods}

\author{Francesco Dinuzzo}

\maketitle

\begin{abstract}
In this paper, we study two general classes of optimization algorithms for kernel methods with convex loss function and quadratic norm regularization, and analyze their convergence. The first approach, based on fixed-point iterations, is simple to implement and analyze, and can be easily parallelized. The second, based on coordinate descent, exploits the structure of additively separable loss functions to compute solutions of line searches in closed form. Instances of these general classes of algorithms are already incorporated into state of the art machine learning software for large scale problems. We start from a solution characterization of the regularized problem, obtained using sub-differential calculus and resolvents of monotone operators, that holds for general convex loss functions regardless of differentiability. The two methodologies described in the paper can be regarded as instances of non-linear Jacobi and Gauss-Seidel algorithms, and are both well-suited to solve large scale problems.
\end{abstract}

\section{Introduction}

The development of optimization software for learning from large datasets is heavily influenced by memory hierarchies of computer storage. In presence of memory constraints, most of the high order optimization methods become unfeasible, whereas techniques such as coordinate descent or stochastic gradient descent may exploit the specific structure of learning functionals to scale well with the dataset size. Considerable effort has been devoted to make kernel methods feasible on large scale problems \citep{Bottou07}. One of the most important features of modern machine learning methodologies is the ability to leverage on sparsity in order to obtain scalability. Typically, learning methods that impose sparsity are based on the minimization of non-differentiable objective functionals. Is this the case of support vector machines or methods based on $\ell_1$ regularization.

In this chapter, we analyze optimization algorithms for a general class of regularization functionals, using sub-differential calculus and resolvents of monotone operators \citep{Rockafellar70, Hiriart-Urruty04} to manage non-differentiability. In particular, we study learning methods that can be interpreted as the minimization of a convex empirical risk term plus a squared norm regularization into a reproducing kernel Hilbert space \citep{Aronszajn50} $\mathcal{H}_K$ with non-null reproducing kernel $K$, namely
\begin{equation}\label{ch02-EQ01}
\min_{g \in \mathcal{H}_K} \left(f\left(g(x_1),\ldots,g(x_{\ell})\right) + \frac{\|g\|^2_{\mathcal{H}_K}}{2}\right),
\end{equation}
\noindent where $f: \mathbb{R}^{\ell}  \rightarrow \mathbb{R}_+$ is a finite-valued bounded below convex function. Regularization problems of the form (\ref{ch02-EQ01}) admit a unique optimal solution which, in view of the representer theorem \citep{Scholkopf01}, can be represented as a finite linear combination of kernel sections:
\[
g(x) =\sum_{i=1}^{\ell}c_iK_{x_i}(x).
\]

We characterize optimal coefficients $c_i$ of the linear combination via a family of non-linear equations. Then, we introduce two general classes of optimization algorithms for large scale regularization methods that can be regarded as instances of non-linear Jacobi and Gauss-Seidel algorithms, and analyze their convergence properties. Finally, we state a theorem that shows how to reformulate convex regularization problems, so as to trade off positive semidefiniteness of the kernel matrix for differentiability of the empirical risk.

\section{Solution characterization}

As a consequence of the representer theorem, an optimal solution of problem (\ref{ch02-EQ01}) can be obtained by solving finite-dimensional optimization problems of the form
\begin{equation}\label{ch02-EQ02}
\min_{c \in \mathbb{R}^{\ell}} F(c), \qquad F(c) = f(\mathbf{K}c)+ \frac{c^T\mathbf{K}c}{2},
\end{equation}
\noindent where $\mathbf{K} \in \mathbb{R}^{\ell \times \ell}$ is a non-null symmetric positive semi-definite matrix called \emph{kernel matrix}. The entries $k_{ij}$ of the kernel matrix are given by
\[
k_{ij} = K(x_i,x_j),
\]
\noindent where $K:\mathcal{X} \times \mathcal{X} \rightarrow \mathbb{R}$ is a positive semidefinite kernel function. It is easy to verify that the resulting kernel matrix is symmetric and positive semi-definite. Let $k_i$ ($i = 1, \ldots, \ell$) denote the columns of the kernel matrix.  Particularly interesting is the case in which function $f$ is additively separable.

\begin{defin}[Additively separable functional]
A functional $f:\mathbb{R}^{\ell} \rightarrow \mathbb{R}$ is called \emph{additively separable} if
\begin{equation}\label{ch02-EQ03}
f(z) = \sum_{i=1}^{\ell}f_i(z_i).
\end{equation}
\end{defin}

Parametric models with $\ell_2$ (ridge) regularization corresponds to the case in which inputs are $n$-dimensional numeric vectors ($\mathcal{X} = \mathbb{R}^{n}$) and the kernel matrix is chosen as $\mathbf{K} =\mathbf{X} \mathbf{X}^T$, where $\mathbf{X} \in \mathbb{R}^{\ell \times n}$ is a matrix whose rows are the input data $x_i$. Letting
\begin{equation}\label{ch02-EQ04}
w := \mathbf{X}^Tc,
\end{equation}
\noindent the following class of  problems is obtained:
\begin{equation}\label{ch02-EQ05}
\min_{w \in \mathbb{R}^n} \left(f(\mathbf{X}w)+ \frac{\|w\|_2^2}{2} \right).
\end{equation}
\noindent Observe that one can optimize over the whole space $\mathbb{R}^n$, since the optimal weight vector will automatically be in the form (\ref{ch02-EQ04}). Parametric models with $\ell_2$ regularization can be seen as specific instances of kernel methods in which $K$ is the linear kernel:
\[
K(x_1,x_2) = \langle x_1,  x_2 \rangle_2.
\]

In the following, two key mathematical objects will be used to characterize optimal solutions of problems (\ref{ch02-EQ02}) and (\ref{ch02-EQ05}). The first is the subdifferential $\partial f$ of the empirical risk. The second is the resolvent of the inverse subdifferential, defined as
\begin{equation}\label{ch02-EQ06}
J_{\alpha} := \left(\mathbf{I}+\alpha \left(\partial f\right)^{-1}\right)^{-1}.
\end{equation}
\noindent See the appendix for more details about these objects. The following result characterizes optimal solutions of problem (\ref{ch02-EQ02}) via a non-linear equation involving $J_{\alpha}$. The characterization also holds for non-differentiable loss functions, and is obtained without introducing constrained optimization problems. The proof of Theorem \ref{ch02-THM01} is given into the appendix.

\begin{thm}\label{ch02-THM01}
For any $\alpha > 0$, there exist optimal solutions of problem (\ref{ch02-EQ02}) such that
\begin{equation}\label{ch02-EQ07}
c = -J_{\alpha}(\alpha  \mathbf{K}c-c),
\end{equation}
\noindent where $J_{\alpha}$ is the resolvent of the inverse sub-differential $\left(\partial f\right)^{-1}$, see (\ref{ch02-EQ06}).
\end{thm}

\noindent The usefulness of condition (\ref{ch02-EQ07}) depends on the possibility of computing closed-form expressions for the resolvent, which may not be feasible for general convex functionals. Remarkably, for many learning methods one can typically exploit the specific structure of $f$ to work out closed-form expressions. For instance, when $f$ is additively separable as in (\ref{ch02-EQ03}), the sub-differential decouples with respect to the different components. In such a case, the computation of the resolvent reduces to the inversion of a function of a single variable, which can be often obtained in closed form. Indeed, in many supervised learning problems, additive separability holds, where $f_i(z_i) = \lambda^{-1} L(y_i,z_i)$,  $L:\mathbb{R} \times \mathbb{R} \rightarrow \mathbb{R}_+$ is a loss function, and $\lambda > 0$ is a regularization parameter. Table \ref{ch02-TAB01} reports the expression of the $J_{\alpha}$ in correspondence with commonly used loss functions. When $f$ is additively separable, the characterization (\ref{ch02-EQ07}) can be generalized as follows.

\begin{cor}\label{ch02-COR01}
Assume that (\ref{ch02-EQ03}) holds. Then, for any $\alpha_i > 0$, $i=1, \ldots, \ell$, there exist optimal solutions of problem (\ref{ch02-EQ02}) such that
\begin{equation}\label{ch02-EQ08}
c_i = -J_{\alpha_i}^i(\alpha_i k_i^T c -c_i), \qquad i=1,\ldots, \ell,
\end{equation}
\noindent where $J_{\alpha_i}^i$ are the resolvents of the inverse sub-differentials $\left(\partial f_i\right)^{-1}$, see (\ref{ch02-EQ06}).
\end{cor}

In this paper,  we analyze two iterative approaches to compute optimal solutions of problem (\ref{ch02-EQ02}), based on the solution characterizations of Theorem \ref{ch02-THM01} and Corollary \ref{ch02-COR01}. For both methods, we show that cluster points of the iteration sequence are optimal solutions, and we have
\begin{equation}\label{ch02-EQ09}
\min_{c \in \mathbb{R}^{\ell}} F(c) = \lim_{k\rightarrow + \infty} F(c^k),
\end{equation}
\noindent where $F$ denote the functional of problem (\ref{ch02-EQ02}). Section \ref{sec02} describes a first approach, which involves simply iterating equation (\ref{ch02-EQ07}) according to the fixed-point method. The method can be also regarded as a non-linear Jacobi algorithm to solve equation (\ref{ch02-EQ07}). It is shown that $\alpha$ can be always chosen so as to make the iterations approximate an optimal solution to arbitrary precision. In section \ref{sec03}, we describe a second approach, that involves separately iterating the single components using the characterization of equation (\ref{ch02-EQ08}). For a suitable choice of $\alpha_i$, the method boils down to coordinate descent, and optimality of cluster points holds whenever indices are picked according to an “essentially cyclical” rule. Equivalently, the method can be regarded as a non-linear Gauss-Seidel algorithm to solve (\ref{ch02-EQ08}).

\begin{table}\label{ch02-TAB01}
  \centering
  \begin{tabular}{|c|c|c|}
    \hline
    Name & Loss $L(y_1,y_2)$ & Operator $-J_{\alpha}(v)$ \\
    \hline
    L1-SVM & $\left(1-y_1y_2\right)_{+}$ & $y \odot \min\left((\alpha \lambda)^{-1},\left(1-y \odot v\right)_{+}\right)$ \\
    L2-SVM & $\left(1-y_1y_2\right)_{+}^2$ & $y \odot \left(1-y \odot v\right)_{+}/(1+\alpha \lambda)$\\
    RLS & $(y_1-y_2)^2/2$ & $\left(y-v\right)/(1+\alpha \lambda)$ \\
    RLA & $|y_1-y_2|$ & $\textrm{sign}(y-v) \odot \min\left((\alpha \lambda)^{-1},|y-v|\right)$ \\
    SVR & $\left(|y_1-y_2|-\epsilon\right)_{+}$ & $\textrm{sign}(y-v)\odot \min\left((\alpha \lambda)^{-1},\left(|y-v|-\epsilon\right)_{+}\right) $ \\
    \hline
  \end{tabular}
  \caption{Operator $-J_{\alpha}$ for different methods. Some of the losses are expressed using the “positive part” function defined as $(x)_{+} = \max\{0,x\}$. In the rightmost column, $\odot$ denotes the element-wise product, and functions are applied component-wise.}
\end{table}

\section{Fixed-point algorithms}  \label{sec02}

In this section, we suggest computing the optimal coefficient vector $c$ of problem (\ref{ch02-EQ02}) by simply iterating equation (\ref{ch02-EQ07}), starting from any initial condition $c^0$:
\begin{equation}\label{ch02-EQ10}
c^{k+1} = -J_{\alpha}(\alpha \mathbf{K}c^{k}-c^{k}).
\end{equation}

\noindent Such procedure is the well-known fixed point iteration (also known as Picard or Richardson iteration) method. Provided that $\alpha$ is properly chosen, the procedure can be used to solve problem (\ref{ch02-EQ02}) to any given accuracy. Before analyzing the convergence properties of method (\ref{ch02-EQ10}), let's study the computational complexity of a single iteration. To this end, one can decompose the iteration into three intermediate steps:
\begin{align*}
z^{k}   & = \mathbf{K}c^{k},                & \textbf{step 1}\\
v^{k}   & = \alpha z^{k}-c^{k},             & \textbf{step 2}\\
c^{k+1} & = -J_{\alpha}(v^{k}).             & \textbf{step 3}
\end{align*}

\noindent The decomposition emphasize the separation between the role of the kernel (affecting only step 1) and the role of the function $f$ (affecting only step 3).

\subsection*{Step 1}

Step one is the only one that involves the kernel matrix. Generally, it is also the most computationally and memory demanding step. Since $z = \mathbf{K}c$ represents predictions on training inputs (or a quantity related to them), it holds that being able to perform fast predictions also have a crucial impact on the training time. This is remarkable, since good prediction speed is a desirable goal on its own. Notice that an efficient implementation of the prediction step is beneficial for any learning method of the form (\ref{ch02-EQ02}), independently of $f$. Ideally, the computational cost of such matrix-vector multiplication is $O(\ell^2)$. However, the kernel matrix might not fit into the memory, so that the time needed to compute the product might also include special computations or additional I/O operations. Observe that, if many components of vector $c$ are null, only a subset of the rows of the kernel matrix is necessary in order to compute the product. Hence, methods that impose sparsity in vector $c$ may produce a significant speed-up in the prediction step. As an additional remark, observe that the matrix-vector product is an operation that can be easily parallelized.

In the linear case (\ref{ch02-EQ05}), the computation of $z^k$ can be divided in two parts:
\begin{align*}
w^{k} & = \mathbf{X}^Tc^{k}, \\
z^{k} & = \mathbf{X}w^{k}.
\end{align*}

\noindent In order to compute the product, it is not even necessary to form the kernel matrix, which may yields a significant memory saving. The two intermediate products both need $O(n\ell)$ operations and the overall cost still scales with $O(n\ell)$. When the number of features is much lower than the number of examples ($n \ll \ell$), there's a significant improvement with respect to $O(\ell^2)$. Speed-up and memory saving are even more dramatic when $\mathbf{X}$ is sparse. In such a case, computing the product in two steps might be more convenient also when $n > \ell$.

\subsection*{Step 2}

Step two is a simple subtraction between vectors, whose computational cost is $O(\ell)$. In section \ref{sec05}, it is shown that $v = \alpha \mathbf{K}c -c$ can be interpreted as the vector of predictions on the training inputs associated with another learning problem consisting in stabilizing a functional regularized whose empirical risk is always differentiable, and whose kernel is not necessarily positive.

\subsection*{Step 3}

Step three is the only one that depends on function $f$. Hence, different algorithms can be implemented by simply choosing different resolvents $J_{\alpha}$. Table \ref{ch02-TAB01} reports the loss function $L$ and the corresponding resolvent for some common supervised learning methods. Some examples are given below. Consider problem (\ref{ch02-EQ02}) with the “hinge” loss function $L(y_1,y_2) = \left(1-y_1y_2\right)_{+}$, associated with the popular Support Vector Machine (SVM). For SVM, step three reads
\[
c^{k+1} = y \odot \min\left(\frac{1}{\alpha \lambda},\left(1-y \odot v^{k}\right)_{+}\right),
\]
\noindent where $\odot$ denotes the element-wise product, and $\min$ is applied element-wise. As a second example, consider classic regularized least squares (RLS). In this case, step three reduces to
\[
c^{k+1} = \frac{y-v^{k}}{1+\alpha \lambda}.
\]
\noindent Generally, the complexity of step three is $O(\ell)$ for any of the classical loss functions.

\subsection{Convergence}

The following result states that the sequence generated by the iterative procedure (\ref{ch02-EQ10}) can be used to approximately solve problem (\ref{ch02-EQ02}) to any precision, provided that $\alpha$ is suitably chosen.

\begin{thm}\label{ch02-THM02}
If the sequence $c^k$ is generated according to algorithm (\ref{ch02-EQ10}), and
\begin{equation}\label{ch02-EQ11}
0 < \alpha < \frac{2}{\|\mathbf{K}\|_{2}},
\end{equation}
\noindent then (\ref{ch02-EQ09}) holds. Moreover, $c^k$ is bounded, and any cluster point is a solution of problem (\ref{ch02-EQ02}).
\end{thm}

A stronger convergence result holds when the kernel matrix is strictly positive or $f$ is differentiable with Lipschitz continuous gradient. Under these conditions, it turns out that the whole sequence $c^k$ converges at least linearly to an unique fixed point.

\begin{thm}\label{ch02-THM03}
Suppose that the sequence $c^k$ is generated according to algorithm (\ref{ch02-EQ10}), where $\alpha$ satisfy (\ref{ch02-EQ11}), and one of the following conditions holds:
\begin{enumerate}
  \item The kernel matrix $\mathbf{K}$ is positive definite.
  \item Function $f$ is everywhere differentiable and $\nabla f$ is Lipschitz continuous,
\end{enumerate}
\noindent Then, there exists a unique solution $c^*$ of equation (\ref{ch02-EQ07}), and $c^k$ converges to $c^*$ with the following rate
\[
\|c^{k+1}-c^*\|_2 \leq \mu \|c^k-c^*\|_2, \qquad 0 \leq \mu < 1.
\]
\end{thm}

\noindent In practice, condition (\ref{ch02-EQ11}) can be equivalently satisfied by fixing $\alpha = 1$ and scaling the kernel matrix to have spectral norm between 0 and 2. In problems that involve a regularization parameter, this last choice will only affect its scale. A possible practical rule to choose the value of $\alpha$ is $\alpha= 1/\|\mathbf{K}\|_{2}$, which is equivalent to scale the kernel matrix to have spectral norm equal to one. However, in order to compute the scaling factor in this way, one generally needs all the entries of the kernel matrix. A cheaper alternative that uses only the diagonal entries of the kernel matrix is $\alpha = 1/\textrm{tr}(\mathbf{K})$, which is equivalent to fix $\alpha$ to one and normalizing the kernel matrix to have trace one. To see that this last rule satisfy condition (\ref{ch02-EQ11}), observe that the trace of a positive semidefinite matrix is an upper bound for the spectral norm. In the linear case (\ref{ch02-EQ05}), one can directly compute $\alpha$ on the basis of the data matrix $\mathbf{X}$. In particular, we have $\|\mathbf{K}\|_{2} = \|\mathbf{X}\|_{2}^2$, and $\textrm{tr}(\mathbf{K}) = \|\mathbf{X}\|_{F}^2$, where $\| \cdot \|_{F}$ denotes the Frobenius norm.

\section{Coordinate-wise iterative algorithms} \label{sec03}

In this section, we describe a second optimization approach that can be seen as a way to iteratively enforce optimality condition (\ref{ch02-EQ08}). Throughout the section, it is assumed that $f$ is additively separable as in (\ref{ch02-EQ03}). In view of Corollary \ref{ch02-COR01}, the optimality condition can be rewritten for a single component as in (\ref{ch02-EQ08}). Consider the following general update algorithm:
\begin{equation}\label{ch02-EQ12}
c_{i}^{k+1} = -J^i_{\alpha_i}(\alpha_i k_{i}^T c^k-c_{i}^k), \qquad i=1,\ldots,\ell.
\end{equation}

\noindent A serial implementation of algorithm (\ref{ch02-EQ10}) can be obtained by choosing $\alpha_i = \alpha$ and by cyclically computing the new components $c_{i}^{k+1}$ according to equation (\ref{ch02-EQ12}). Observe that this approach requires to keep in memory both $c^k$ and $c^{k+1}$ at a certain time. In the next sub-section, we analyze a different choice of parameters $\alpha_i$ that leads to a class of coordinate descent algorithms, based on the principle of using new computed information as soon as it is available.

\subsection{Coordinate descent methods}

\begin{algorithm}
\caption{Coordinate descent for regularized kernel methods} \label{ch02-ALG01}
\begin{algorithmic}
\WHILE{$\max_{i} |h_i| \geq \delta$}
\STATE Pick a coordinate index $i$ according to some rule,  \\
\STATE $z_i^k = k_{i}^T c^k$,                               \\
\STATE $v_i^k = z_i^k/k_{ii}-c_i^k$,                        \\
\STATE $\textrm{tmp} = S_i(v_i^k)$,                         \\
\STATE $h_i = \textrm{tmp} -c_i^k$,                         \\
\STATE $c_i^{k+1} = \textrm{tmp}$,                          \\
\ENDWHILE
\end{algorithmic}
\end{algorithm}

A coordinate descent algorithm updates a single variable at each iteration by solving a sub-problem of dimension one. During the last years, optimization via coordinate descent is becoming a popular approach in machine learning and statistics, since its implementation is straightforward and enjoys favorable computational properties \citep{Friedman07, Tseng08, Wu08, Chang08, Hsieh08, Yun09, Huang10, Friedman10}. Although the method may require many iterations to converge, the specific structure of supervised learning objective functionals allows to solve the sub-problems with high efficiency. This makes the approach competitive especially for large-scale problems, in which memory limitations hinder the use of second order optimization algorithms. As a matter of fact, state of the art solvers for large scale supervised learning such as \texttt{glmnet} \citep{Friedman10} for generalized linear models, or \texttt{LIBLINEAR} \citep{Fan08} for SVMs are based on coordinate descent techniques.

The update for $c_i^k$ in algorithm (\ref{ch02-EQ12}) also depends on components $c_j^k$ with $j < i$ which have already been updated. Hence, one needs to keep in memory coefficients from two subsequent iterations $c^{k+1}$ and $c^{k}$. In this sub-section, we describe a method that allows to take advantage of the computed information as soon as it is available, by overwriting the coefficients with the new values. Assume that the diagonal elements of the kernel matrix are strictly positive, i.e. $k_{ii} > 0$. Notice that this last assumption can be made without any loss of generality. Indeed, if $k_{ii} = 0$ for some index $i$ then, in view of the inequality $|k_{ij}| \leq \sqrt{k_{ii} k_{jj}}$, it follows that $k_{ij} = 0$ for all $j$. Hence, the whole $i$-th column (row) of the kernel matrix is zero, and can be removed without affecting optimization results for the other coefficients. By letting $\alpha_i = 1/k_{ii}$ and $S_i := -J^i_{(k_{ii})^{-1}}$ in equation (\ref{ch02-EQ08}), the $i$-th coefficient in the inner sum does cancel out, and we obtain
\begin{equation}\label{ch02-EQ13}
c_{i} = S_i\left(\sum_{j \neq i} \frac{k_{ij}}{k_{ii}}c_j\right).
\end{equation}

\noindent The optimal $i$-th coefficient is thus expressed as a function of the others. Similar characterizations have been also derived in \citep{Dinuzzo09} for several loss functions. Equation (\ref{ch02-EQ13}) is the starting point to obtain a variety of coordinate descent algorithms involving the iterative choice of a a coordinate index $i$ followed by the optimization of $c_i$ as a function of the other coefficients. A simple test on the residual of equation (\ref{ch02-EQ13}) can be used as a stopping condition. The approach can be also regarded as a non-linear Gauss-Seidel method \citep{Ortega00} for solving the equations (\ref{ch02-EQ08}). It is assumed that vector $c$ is initialized to some initial $c^0$, and coefficients $h_i$ are initialized to the residuals of equation (\ref{ch02-EQ13}) evaluated in correspondence with $c^0$. Remarkably, in order to implement the method for different loss functions, we simply need to modify the expression of functions $S_i$. Each update only involves a single row (column) of the kernel matrix. In the following, we will assume that indices are recursively picked according to a rule that satisfy the following condition, see \citep{Tseng01, Luenberger08}.

\paragraph{Essentially Cyclic Rule.} There exists a constant integer $T > \ell$ such that every index $i \in \{1, \ldots , \ell \}$ is chosen at least once between the $k$-th iteration and the $(k+T-1)$-th, for all $k$.\\

Iterations of coordinate descent algorithms that use an essentially cyclic rule can be grouped in \emph{macro-iterations}, containing at most $T$ updates of the form (\ref{ch02-EQ13}), within which all the indices are picked at least once. Below, we report some simple rules that satisfy the essentially cyclic condition and don't require to maintain any additional information (such as the gradient):
\begin{enumerate}
  \item \textbf{Cyclic rule:} In each macro-iteration, each index is picked exactly once in the order $1, \ldots, \ell$. Hence, each macro-iteration consists exactly of $\ell$ iterations.
  \item \textbf{Aitken double sweep rule:} Consists in alternating macro-iterations in which indices are chosen in the natural order $1, \ldots, \ell$ with macro-iterations in the reverse order, i.e. $(\ell-1), \ldots, 1$.
  \item \textbf{Randomized cyclic rule:} The same as the cyclic rule, except that indices are randomly permuted at each macro-iteration.
\end{enumerate}

\noindent In the linear case (\ref{ch02-EQ05}), $z_i^k$ can be computed as follows
\begin{align*}
w^k         & = \mathbf{X} c^k,  \\
z_i^k       & = x_i^T w^k.
\end{align*}
\noindent By exploiting the fact that only one component of vector $c$ changes from an iteration to the next, the first equation can be further developed:
\[
w^k = \mathbf{X}^T c^k  = w^{k-1}+ (\mathbf{X}^T e_{p}) h_{p} = w^{k-1} + x_{p} h_{p}
\]
\noindent where $p$ denotes the index chosen in the previous iteration, and $h_{p}$ denotes the variation of coefficient $c_p$ in the previous iteration. By introducing these new quantities, the coordinate descent algorithm can be rewritten as in Algorithm \ref{ch02-ALG02}, where we have set $S_i := -J^i_{\|x_i\|_2^{-2}}$.

\begin{algorithm}
\caption{Coordinate descent (linear kernel)} \label{ch02-ALG02}
\begin{algorithmic}
\WHILE{$\max_{i} |h_i| \geq \delta$}
\STATE Pick a coordinate index $i$ according to some rule,          \\
\IF{$h_p \neq 0$}
\STATE $w^k  = w^{k-1}+ x_{p} h_{p}$,                       \\
\ENDIF
\STATE $z_i^k = x_i^T w^k$,                                 \\
\STATE $v_i^k = z_i^k/\|x_i\|_2^2-c_i^k$,                     \\
\STATE $\textrm{tmp} = S_i(v_i^{k})$,                       \\
\STATE $h_i = \textrm{tmp} -c_i^k$,                         \\
\STATE $c_i^{k+1} = \textrm{tmp}$,                          \\
\STATE $p = i$
\ENDWHILE
\end{algorithmic}
\end{algorithm}

The computational cost of a single iteration depends mainly on the updates for $w$ and $z_i$, and scales linearly with the number of features, i.e. $O(n)$. When the loss function have linear traits, it is often the case that coefficient $c_i$ doesn't change after the update, so that $h_i = 0$. When this happen, the next update of $w$ can be skipped, obtaining a significant speed-up. Further, if the vectors $x_i$ are sparse, the average computational cost of the second line may be much lower than $O(n)$. A technique of this kind has been proposed in \citep{Hsieh08} and implemented in the package \texttt{LIBLINEAR} \citep{Fan08} to improve speed of coordinate descent iterations for linear SVM training. Here, one can see that the same technique can be applied to any convex loss function, provided that an expression for the corresponding resolvent is available.

The main convergence result for coordinate descent is stated below. It should be observed that the classical theory of convergence for coordinate descent is typically formulated for differentiable objective functionals. When the objective functional is not differentiable, there exist counterexamples showing that the method may get stuck in a non-stationary point \citep{Auslender76}. In the non-differentiable case, optimality of cluster points of coordinate descent iterations has been proven in \citep{Tseng01} (see also references therein), under the additional assumption that the non-differentiable part is additively separable. Unfortunately, the result of \citep{Tseng01} cannot be directly applied to problem (\ref{ch02-EQ02}), since the (possibly) non-differential part $f(\mathbf{K}c)$ is not separable with respect to the optimization variables $c_i$, even when (\ref{ch02-EQ03}) holds. Notice also that, when the kernel matrix is not strictly positive, level sets of the objective functional are unbounded (see Lemma \ref{chA2-LEM01} in the appendix). Despite these facts, it still holds that cluster points of coordinate descent iterations are optimal, as stated by the next Theorem.

\begin{thm}\label{ch02-THM04}
Suppose that the following conditions hold:
\begin{enumerate}
  \item Function $f$ is additively separable as in (\ref{ch02-EQ03}),
  \item The diagonal entries of the kernel matrix satisfy $k_{ii} > 0$,
  \item The sequence $c^k$ is generated by the coordinate descent algorithm (Algorithm \ref{ch02-ALG01} or \ref{ch02-ALG02}), where indices are recursively selected according to an essentially cyclic rule.
\end{enumerate}
\noindent Then, (\ref{ch02-EQ09}) holds, $c^k$ is bounded, and any cluster point is a solution of problem (\ref{ch02-EQ02}).
\end{thm}

\section{A reformulation theorem}  \label{sec05}

The following result shows that solutions of problem (\ref{ch02-EQ02}) satisfying equation (\ref{ch02-EQ08}) are also stationary points of a suitable family of differentiable functionals.

\begin{thm}\label{ch02-THM05}
If $c$ satisfy (\ref{ch02-EQ07}), then it is also a stationary point of the following functional:
\[
F_{\alpha}(c) = \alpha^{-1} f_{\alpha}(\mathbf{K}_{\alpha}c)+\frac{c^T\mathbf{K}_{\alpha} c}{2},
\]
\noindent where $f_{\alpha}$ denotes the Moreau-Yosida regularization of $f$, and $\mathbf{K}_{\alpha} := \alpha \mathbf{K}- \mathbf{I}$.
\end{thm}

\noindent Theorem \ref{ch02-THM05} gives an insight into the role of parameter $\alpha$, as well as providing an interesting link with machine learning with indefinite kernels. By the properties of the Moreau-Yosida regularization, $f_{\alpha}$ is differentiable with Lipschitz continuous gradient. It follows that $F_{\alpha}$ also have such property. Notice that lower values of $\alpha$ are associated with smoother functions $f_{\alpha}$, while the gradient of $\alpha^{-1}f_{\alpha}$ is non-expansive. A lower value of $\alpha$ also implies a “less positive semidefinite” kernel, since the eigenvalues of $\mathbf{K}_{\alpha}$ are given by $(\alpha\alpha_i-1)$, where $\alpha_i$ denote the eigenvalues of $\mathbf{K}$. Indeed, the kernel becomes non-positive as soon as $\alpha \min_{i}\{\alpha_i\} < 1$. Hence, the relaxation parameter $\alpha$ regulates a trade-off between smoothness of $f_{\alpha}$ and positivity of the kernel.

When $f$ is additively separable as in (\ref{ch02-EQ03}), it follows that $f_{\alpha}$ is also additively separable:
\[
f_{\alpha}(z) = \sum_{i=1}^{\ell} f_{i\alpha}(z_i),
\]
\noindent and $f_{i\alpha}$ is the Moreau-Yosida regularization of $f_i$. The components can be often computed in closed form, so that an “equivalent differentiable loss function” can be derived for non-differentiable problems. For instance, when $f_i$ is given by the hinge loss $f_i(z_i) = \left(1-y_i z_i\right)_{+}$, letting $\alpha = 1$, we obtain
\[
f_{i 1}(z_i) = \left\{
                 \begin{array}{ll}
                  1/2-y_iz_i, & y_iz_i \leq 0\\
                  (1-y_iz_i)_{+}^2/2, & y_iz_i > 0
                 \end{array}
               \right.
\]

\noindent Observe that this last function is differentiable with Lipschitz continuous derivative. By Theorem \ref{ch02-THM05}, it follows that the SVM solution can be equivalently computed by searching the stationary points of a new regularization functional obtained by replacing the hinge loss with its equivalent differentiable loss function, and modifying the kernel matrix by subtracting the identity.

\section{Conclusions}

In this paper, fixed-point and coordinate descent algorithms for regularized kernel methods with convex empirical risk and squared RKHS norm regularization have been analyzed. The two approaches can be regarded as instances of non-linear Jacobi and Gauss-Seidel algorithms to solve a suitable non-linear equation that characterizes optimal solutions. While the fixed-point algorithm has the advantage of being parallelizable, the coordinate descent algorithm is able to immediately exploit the information computed during the update of a single coefficient. Both classes of algorithms have the potential to scale well with the dataset size. Finally, it has been shown that minimizers of convex regularization functionals are also stationary points of a family of differentiable regularization functionals involving the Moreau-Yosida regularization of the empirical risk.

\appendix

\section*{Appendix A}

In this section, we review some concepts and theorems from analysis and linear algebra, which are used in the proofs. Let $\mathbb{E}$ denote an Euclidean space endowed with the standard inner product $\langle x_1, x_2 \rangle_2 = x_1^T x_2$ and the induced norm $\|x\|_2 = \sqrt{\langle x,x \rangle_2}$.

\subsection*{Set-valued maps}

A set-valued map (or multifunction) $A:\mathbb{E} \rightarrow 2^{\mathbb{E}}$ is a rule that associate to each point $x \in \mathbb{E}$ a subset $A(x) \subseteq \mathbb{E}$. Notice that any map $A:\mathbb{E}\rightarrow \mathbb{E}$ can be seen as a specific instance of multifunction such that $A(x)$ is a singleton for all $x \in \mathbb{E}$. The multi-function $A$ is called \emph{monotone} whenever
\[
\langle y_1-y_2,  x_1-x_2 \rangle_2 \geq 0, \qquad \forall x_1, x_2 \in \mathbb{E}, \qquad y_1 \in A(x_1), \quad y_2 \in A(x_2),
\]
\noindent If there exists $L \geq 0$ such that
\[
\|y_1-y_2\|_2 \leq L \|x_1-x_2\|_2, \qquad \forall x_1, x_2 \in \mathbb{E}, \qquad y_1 \in A(x_1), \quad y_2 \in A(x_2),
\]
\noindent then $A$ is single-valued, and is called \emph{Lipschitz continuous function} with modulus $L$. A Lipschitz continuous function is called \emph{nonexpansive} if $L = 1$, \emph{contractive} if $L < 1$, and \emph{firmly non-expansive} if
\[
\|y_1-y_2\|_2^2 \leq \langle y_1-y_2, x_1-x_2 \rangle_2 , \qquad \forall x_1, x_2 \in \mathbb{E}, \qquad y_1 \in A(x_1), \quad y_2 \in A(x_2).
\]
\noindent In particular, firmly non-expansive maps are single-valued, monotone, and non-expansive. For any monotone multifunction $A$, its \emph{resolvent} $J^A_{\alpha}$ is defined for any $\alpha > 0$ as $J^A_{\alpha} := \left(\mathbf{I}+\alpha A\right)^{-1}$, where $\mathbf{I}$ stands for the identity operator. Resolvents of monotone operators are known to be firmly non-expansive.

\subsection*{Finite-valued convex functions}

A function $f:\mathbb{E} \rightarrow \mathbb{R}$ is called \emph{finite-valued convex} if, for any $\alpha \in [0, 1]$ and any $x_1, x_2 \in \mathbb{E}$, it satisfy
\[
-\infty < f(\alpha x_1 +(1-\alpha)x_2) \leq \alpha f(x_1) +(1-\alpha) f(x_2) < +\infty
\]
\noindent The subdifferential of a finite-valued convex function $f$ is a multifunction $\partial f: \mathbb{E} \rightarrow 2^{\mathbb{E}}$ defined as
\[
\partial f(x) = \left\{\xi \in \mathbb{E}: f(y)-f(x) \geq \langle \xi,  y-x \rangle_2, \quad \forall y \in \mathbb{E} \right\}.
\]
\noindent It can be shown that the following properties hold:
\begin{enumerate}
  \item $\partial f(x)$ is a non-empty convex compact set for any $x \in \mathbb{E}$.
  \item $f$ is (Gâteaux) differentiable at $x$ if and only if $\partial f(x) = \{\nabla f(x)\}$ is a singleton (whose unique element is the gradient).
  \item $\partial f$ is a monotone multifunction.
  \item The point $x^*$ is a (global) minimizer of $f$ if and only if $0 \in \partial f(x^*)$.
\end{enumerate}

\noindent For any finite-valued convex function $f$, its Moreau-Yosida regularization (or Moreau envelope, or quadratic min-convolution) is defined as
\[
f_{\alpha}(x) := \min_{y \in \mathbb{E}}\left(f(y)+\frac{\alpha}{2}\|y-x\|_2^2\right).
\]
\noindent For any fixed $x$, the minimum in the definition of $f_{\alpha}$ is attained at $y = p_{\alpha}(x)$, where $p_{\alpha} := \left(\mathbf{I}+\alpha^{-1} \partial f\right)^{-1}$ denotes the so-called \emph{proximal mapping}. It can be shown that the following remarkable properties hold:
\begin{enumerate}
 \item $f_{\alpha}$ is convex differentiable, and the gradient $\nabla f_{\alpha}$ is Lipschitz continuous with modulus $1/\alpha$.
 \item $f_{\alpha}(x) = f(p_{\alpha}(x))+\frac{\alpha}{2}\|p_{\alpha}(x)-x\|_2^2$.
 \item $f_{\alpha}$ and $f$ have the same set of minimizers for all $\alpha$.
 \item The gradient $\nabla f_{\alpha}$ is called Moreau-Yosida regularization of $\partial f$, and satisfy
        \[
        \nabla f_{\alpha}(x) = \alpha \left(x-p_{\alpha}(x)\right) = \alpha J_{\alpha}(x),
        \]
    \noindent where $J_{\alpha}$ denote the resolvent of the inverse sub-differential defined as
    \[
    J_{\alpha} := \left(\mathbf{I}+\alpha \left(\partial f\right)^{-1}\right)^{-1}.
    \]
\end{enumerate}

\subsection*{Convergence theorems}

\begin{thm}[Contraction mapping theorem]\label{chA1-THM01}
Let $A: \mathbb{E} \rightarrow \mathbb{E}$ and suppose that, given $c^0$, the sequence $c^k$ is generated as
\[
c^{k+1} = A(c^k).
\]
\noindent If $A$ is contractive with modulus $\mu$, then there exists a unique fixed-point $c^*$ such that $c^* = A(c^*)$, and the sequence $c^k$ converges to $c^*$ at linear rate:
\[
\|c^{k+1}-c^*\|_2 \leq \mu \|c^k-c^*\|_2, \qquad 0 \leq \mu < 1.
\]
\end{thm}

\medskip

\noindent The following result is know as Zangwill's convergence theorem \citep{Zangwill69}, see also page 206 of \citep{Luenberger08}.
\begin{thm}[Zangwill's convergence theorem]\label{chA1-THM02}
Let $A: \mathbb{E} \rightarrow 2^{\mathbb{E}}$ denote a multifunction, and suppose that, given $c^0$, the sequence $c^k$ is generated as
\[
c^{k+1} \in A(c^k).
\]
\noindent Let $\Gamma \subset \mathbb{E}$ called \emph{solution set}. If the following conditions hold:
\begin{enumerate}
  \item The graph $G_A = \left\{\left(x,y\right) \in \mathbb{E} \times \mathbb{E}:  y \in A(x) \right\}$ is a closed set,
  \item There exists a \emph{descent function} $F$ such that
    \begin{itemize}
    \item For all $x \in \Gamma$, $F(A(x)) \leq F(x)$,
    \item For all $x \notin \Gamma$, $F(A(x)) < F(x)$,
    \end{itemize}
  \item The sequence $c^k$ is bounded,
\end{enumerate}
\noindent then all the cluster points of $c^k$ belongs to the solution set.
\end{thm}

\section*{Appendix B}

The following Lemma will prove useful in the subsequent proofs.

\begin{lem}\label{chA2-LEM01}
The functional $F$ of problem (\ref{ch02-EQ02}) is such that $F(c+u) = F(c)$, for any vector $u$ in the nullspace of the kernel matrix.
\end{lem}
\begin{proof}
Let $u$ denote any vector in the nullspace of the kernel matrix. Then, we have
\[
F(c+u) = f\left(\mathbf{K}(c+u)\right)+ \frac{(c+u)^{T}\mathbf{K}(c+u)}{2} =  f\left(\mathbf{K}c\right)+ \frac{c^{T}\mathbf{K}c}{2} = F(c).
\]
\end{proof}

\begin{proof}[Proof of Theorem \ref{ch02-THM01}]
Problem (\ref{ch02-EQ02}) is a convex optimization problem, where the functional $F$ is continuous and bounded below. First of all, we show that there exists optimal solution. Observe that minimization can be restricted to the range of the kernel matrix. Indeed, any vector $c \in \mathbb{E}$ can be uniquely decomposed as $c= u+v$, where $u$ belongs to the nullspace of $\mathbf{K}$ and $v$ belongs to the range. By Lemma \ref{chA2-LEM01}, we have $F(c) = F(v)$. Since $F$ is coercive on the range of the kernel matrix ($\lim_{\|v\|_2\rightarrow +\infty} F(v) = +\infty$), it follows that there exist optimal solutions.

A necessary and sufficient condition for $c^*$ to be optimal is
\[
0 \in \partial F(c^*) = \mathbf{K} \left(\partial f \left(\mathbf{K}c^*\right) + c^*\right) = \mathbf{K} G(c^*), \qquad G(c^*) := \partial f \left(\mathbf{K}c^*\right) + c.
\]
\noindent Consider the decomposition $G(c^*) = u_{G} + v_{G}$, where $u_{G}$ belongs to the nullspace of the kernel matrix and $v_{G}$ belongs to the range. Observe that
\[
v_{G} = G(c^*)-u_{G} = G(c^*-u_{G}).
\]
\noindent We have
\[
0 \in \mathbf{K} G(c^*) = \mathbf{K} v_{G} \quad \Rightarrow \quad 0 \in G(c^*-u_{G}) = v_{G},
\]
\noindent so that, for any optimal $c^*$, there exists an optimal $c = c^*-u_{G}$ such that
\begin{equation}\label{chA2-EQ01}
0 \in  \partial f \left(\mathbf{K}c\right) +  c.
\end{equation}
\noindent By introducing the inverse sub-differential, equation (\ref{chA2-EQ01}) can be rewritten as
\[
\mathbf{K}c \in \left(\partial f\right)^{-1}(-c).
\]
\noindent Multiplying by $\alpha > 0$ both sides and subtracting $c$, we obtain
\[
\alpha \mathbf{K}c - c \in \alpha \left(\partial f\right)^{-1}(-c)- c.
\]
\noindent Finally, introducing the resolvent $J_{\alpha}$ as in (\ref{ch02-EQ06}), we have
\[
\alpha \mathbf{K}c - c \in \left(J_{\alpha}\right)^{-1}(- c)
\]
\noindent Since $J_{\alpha}$ is single-valued, equation (\ref{ch02-EQ07}) follows.
\end{proof}

\begin{proof}[Proof of Corollary \ref{ch02-COR01}]
Let's start from the sufficient condition for optimality (\ref{chA2-EQ01}). If (\ref{ch02-EQ03}) holds, then the subdifferential of $f$ decouples with respect to the different components, so that there exist optimal coefficients $c_i$ such that
\[
0 \in \partial f_i \left(k_i^T c\right) + c_i, \qquad i=1,\ldots,\ell.
\]
\noindent Equivalently,
\[
k_i^T c \in \left(\partial f_i\right)^{-1}(-c_i).
\]
\noindent Multiplying by $\alpha_i > 0$ both sides and subtracting $c_i$, we have
\[
\alpha_i k_i^T c -c_i\in \alpha_i \left(\partial f_i\right)^{-1}(-c_i) -c_i.
\]
\noindent The thesis follows by introducing the resolvents $J^i_{\alpha_i}$ and solving for $-c_i$.
\end{proof}

\begin{proof}[Proof of Theorem \ref{ch02-THM02}]

We show that the sequence $c^k$ generated by algorithm (\ref{ch02-EQ10}) converges to an optimal solution of Problem (\ref{ch02-EQ02}). By Theorem \ref{ch02-THM01}, there exists optimal solutions $c^*$ satisfying (\ref{ch02-EQ07}). We now observe that any other vector $c$ such that $\mathbf{K}(c^*-c) = 0$ is also optimal. Indeed, we have $c = c^*+u$, where $u$ belongs to the nullspace of the kernel matrix. By Lemma \ref{chA2-LEM01}, it follows that $F(c) = F(c^*)$. To prove (\ref{ch02-EQ09}), it suffices to show that $\mathbf{K}r^k \rightarrow 0$, where $r^k := c^k-c^*$ can be uniquely decomposed as
\[
r^k = u^k+v^k, \qquad \mathbf{K}u^k = 0, \qquad \langle u^{k}, v^k \rangle_2 = 0.
\]
\noindent We need to prove that $\|v^k\|_2 \rightarrow 0$. Since $J_{\alpha}$ is nonexpansive, we have
\begin{align*}
\gamma^{k+1} :=     \|r^{k+1}\|_2^2 &  = \|c^{k+1}-c^*\|_2^2 \\
            & =     \|J_{\alpha}(\alpha \mathbf{K}c^{k}-c^{k})-J_{\alpha} (\alpha \mathbf{K}c^*-c^*)\|_2^2 \\
            & \leq  \|\alpha \mathbf{K}r^{k}-r^{k}\|_2^2\\
            & =     \|\alpha \mathbf{K}v^{k}-r^{k}\|_2^2.
\end{align*}

\noindent Observing that $v^k$ is orthogonal to the nullspace of the kernel matrix, we can further estimate as follows
\[
\|\alpha \mathbf{K}v^k-r^k\|_2^2 = \gamma^k-v^{kT}\left(2\alpha \mathbf{K}-\alpha^2\mathbf{K}^2\right)v^j \leq \gamma^k-\beta \|v^k\|_2^2,
\]
\noindent where
\[
\beta := \min_{i:\alpha_i> 0}\alpha \alpha_i(2-\alpha \alpha_i).
\]
\noindent and $\alpha_i$ denote the eigenvalues of the kernel matrix. Since the kernel matrix is positive semidefinite and condition (\ref{ch02-EQ11}) holds, we have
\[
0 \leq \alpha\alpha_i < 2.
\]
\noindent Since the kernel matrix is not null and have a finite number of eigenvalues, there's at least one eigenvalue with strictly positive distance from zero. It follows that $\beta > 0$. Since
\[
0 \leq \gamma^{k+1}  \leq \gamma^0 -\beta \sum_{j=1}^{k}\|v^j\|_2^2,
\]
\noindent we have, necessarily, that $\|v^k\|_2 \rightarrow 0$. Finally, observe that $c^k$ remains bounded
\[
\|c^k\|_2 \leq \|r^k\|_2+\|c^*\|_2 \leq \|r^0\|_2+\|c^*\|_2,
\]
\noindent so that there's a subsequence converging to an optimal solution. In fact, by (\ref{ch02-EQ09}) it follows that any cluster point of $c^k$ is an optimal solution.
\end{proof}

\begin{proof}[Proof of Theorem \ref{ch02-THM03}]
Algorithm (\ref{ch02-EQ10}) can be rewritten as
\[
c^{k+1} = A(c^k),
\]
\noindent where the map $A:\mathbb{E} \rightarrow \mathbb{E}$ is defined as
\[
A(c) := -J_{\alpha}\left(\alpha \mathbf{K}c-c\right).
\]

\noindent Under both conditions (1) and (2) of the theorem, we show that $A$ is contractive. Uniqueness of the fixed-point, and convergence with linear rate will then follow from the contraction mapping theorem (Theorem \ref{chA1-THM01}). Let
\[
\mu_1 := \| \alpha \mathbf{K} -\mathbf{I}\|_2 = \max_{i}|1-\alpha_i\alpha|,
\]
\noindent where $\alpha_i$ denote the eigenvalues of the kernel matrix. Since the kernel matrix is positive semidefinite, and condition (\ref{ch02-EQ11}) holds, we have
\[
0 \leq \alpha_i\alpha < 2,
\]
\noindent so that $\mu_1 \leq 1$. We now show that the following inequality holds:
\begin{equation}\label{chA2-EQ02}
\| J_{\alpha}(y_1)-J_{\alpha}(y_2)\|_2 \leq \mu_2 \|y_1-y_2\|_2,
\end{equation}
\noindent where
\[
\mu_2 := \left(1+\frac{1}{L^2}\right)^{-1/2},
\]
\noindent and $L$ denotes the Lipschitz modulus of $\nabla f$ when $f$ is differentiable with Lipschitz continuous gradient, and $L=+\infty$ otherwise. Since $J_{\alpha}$ is nonexpansive, it is easy to see that (\ref{chA2-EQ02}) holds when $L = +\infty$. Suppose now that $f$ is differentiable and $\nabla f$ is Lipschitz continuous with modulus $L$. It follows that the inverse gradient satisfies
\[
\|(\nabla f)^{-1}(x_1)-(\nabla f)^{-1}(x_2)\|_2 \geq \frac{1}{L} \|x_1-x_2\|_2.
\]
\noindent Since $(\nabla f)^{-1}$ is monotone, we have
\begin{align*}
\| J_{\alpha}^{-1}(x_1)-J_{\alpha}^{-1}(x_2)\|_2^2     & = \|x_1-x_2 + (\nabla f)^{-1}(x_1)-(\nabla f)^{-1}(x_2)\|_2^2 \\
                                                & \geq \|x_1-x_2\|_2^2+ \|(\nabla f)^{-1}(x_1)-(\nabla f)^{-1}(x_2)\|_2^2\\
                                                & \geq\left(1+\frac{1}{L^2}\right)\|x_1-x_2\|_2^2.
\end{align*}

\noindent From this last inequality, we obtain (\ref{chA2-EQ02}). Finally, we have
\begin{align*}
\|A(c_1)-A(c_2)\|_2     & = \|J_{\alpha}\left(\alpha \mathbf{K}c_1-c_1\right)-J_{\alpha}\left(\alpha \mathbf{K}c_2-c_2\right)\|_2 \\
                        &\leq \mu_2 \| (\alpha \mathbf{K} -\mathbf{I}) (c_1-c_2) \|_2 \\
                        &  \leq \mu \|c_1-c_2\|_2,
\end{align*}
\noindent where we have set $\mu := \mu_1 \mu_2$. Consider the case in which $\mathbf{K}$ is strictly positive definite. Then, it holds that
\[
0 < \alpha_i\alpha < 2,
\]
\noindent so that $\mu_1 < 1$, and $A$ is contractive. Finally, when $f$ is differentiable and $\nabla f$ is Lipschitz continuous, we have $\mu_2 < 1$ and, again, it follows that $A$ is contractive. By the contraction mapping theorem (Theorem \ref{chA1-THM01}), there exists a unique $c^*$ satisfying (\ref{ch02-EQ07}), and the sequence $c^k$ of Picard iterations converges to $c^*$ at a linear rate.
\end{proof}

\begin{proof}[Proof of Theorem \ref{ch02-THM04}]
We shall apply Theorem \ref{chA1-THM02} to the coordinate descent macro-iterations, where the solution set $\Gamma$ is given by
\[
\Gamma := \left\{c \in \mathbb{E}: (\ref{ch02-EQ08}) \quad \textrm{holds}\right\}.
\]
\noindent Let $A$ denote the algorithmic map obtained after each macro-iteration of the coordinate descent algorithm. By the essentially cyclic rule, we have
\[
c \in A(c) = \bigcup_{\left(i_1, \ldots, i_s\right) \in I}\left\{\left(A_{i_1} \circ \cdots \circ A_{i_s}\right) (c)\right\},
\]
\noindent where $I$ is the set of strings of length at most $s = T$ on the alphabet $\{1, \ldots, \ell\}$ such that all the characters are picked at least once. Observing that the set $I$ has finite cardinality, it follows that the graph $G_A$ is the union of a finite number of graphs of point-to-point maps:
\[
G_A = \bigcup_{\left(i_1, \ldots, i_s\right) \in I} \left\{(x,y) \in \mathbb{E} \times \mathbb{E}: y = \left( A_{i_1} \circ \cdots \circ A_{i_s}\right) (x)\right\}.
\]
\noindent Now notice that each map $A_i$ is of the form
\[
A_i(c) = c+e_i t_i(c), \qquad t_i(c) := S_i\left(\sum_{j \neq i} \frac{k_{ij}}{k_{ii}}c_j\right)-c_i.
\]
\noindent All the resolvents are Lipschitz continuous, so that functions $A_i$ are also Lipschitz continuous. It follows that the composition of a finite number of such maps is continuous, and its graph is a closed set. Since the union of a finite number of closed sets is also closed, we obtain that $G_A$ is closed.

Each map $A_i$ yields the solution of an exact line search over the $i$-th coordinate direction for minimizing functional $F$ of Problem (\ref{ch02-EQ02}). Hence, the function
\[
\phi_i(t) = F(c+e_it),
\]
\noindent is minimized at $t_i(c)$, that is
\[
0 \in \partial \phi_i (t_i(c)) = \langle e_i, \partial F(c+e_i t_i(c)) \rangle_2 = \langle k_i, \partial f(\mathbf{K}A_i(c))+A_i(c) \rangle_2.
\]
\noindent Equivalently,
\begin{equation}\label{chA2-EQ03}
-\langle k_i,  A_i(c) \rangle_2 \in  \langle k_i, \partial f(\mathbf{K}A_i(c)) \rangle_2.
\end{equation}
\noindent By definition of subdifferential, we have
\[
f(\mathbf{K}A_i(c))-f(\mathbf{K}c) \leq t_i(c) \gamma, \qquad \forall \gamma \in  \langle k_i, \partial f(\mathbf{K}A_i(c))\rangle_2.
\]
\noindent In particular, in view of (\ref{chA2-EQ03}), we have
\[
f(\mathbf{K}A_i(c))-f(\mathbf{K}c) \leq -t_i(c) \langle k_i, A_i(c) \rangle_2.
\]

\noindent Now, observe that
\begin{align*}
F(A(c)) & \leq F(A_i(c)) =  F(c+e_i t_i(c)) \\
        & = F(c) + t_i^2(c) \frac{k_{ii}}{2}  + t_i(c) \langle k_i, c \rangle_2+ f(\mathbf{K}A_i(c))-f(\mathbf{K}c)\\
        & \leq  F(c) + t_i^2(c) \frac{k_{ii}}{2}  + t_i(c) \langle k_i, c\rangle_2 - t_i(c) \langle k_i, A_i(c)\rangle_2\\
        & =  F(c) + t_i^2(c) \frac{k_{ii}}{2}  + t_i(c) \langle k_i, c-A_i(c)\rangle_2\\
        & =  F(c) + t_i^2(c) \frac{k_{ii}}{2}  - t_i^2(c) k_{ii}\\
        & =  F(c) - t_i^2(c) \frac{k_{ii}}{2}.
\end{align*}
\noindent Since $k_{ii} > 0$, the following inequalities hold:
\begin{equation}\label{chA2-EQ04}
t_i^2(c) \leq \frac{2}{k_{ii}}\left(F(c)-F(A_i(c))\right) \leq \frac{2}{k_{ii}}\left(F(c)-F(A(c))\right).
\end{equation}

\noindent We now show that $F$ is a \emph{descent function} for the map $A$ associated with the solution set $\Gamma$. Indeed, if $c$ satisfy (\ref{ch02-EQ08}), then the application of the map $A$ doesn't change the position, so that
\[
F(A(c)) = F(c).
\]
\noindent On the other hand, if $c$ does not satisfy (\ref{ch02-EQ08}), there's at least one index $i$ such that $t_i(c) \neq 0$. Since all the components are chosen at least once, and in view of (\ref{chA2-EQ04}), we have
\[
F(A(c)) < F(c).
\]

\noindent Finally, we need to prove that the sequence of macro-iterations remains bounded. In fact, it turns out that the whole sequence $c^{k}$ of iterations of the coordinate descent algorithm is bounded. From the first inequality in (\ref{chA2-EQ04}), the sequence $F(c^k)$ is non-increasing and bounded below, and thus it must converge to a number
\begin{equation}\label{chA2-EQ05}
F_{\infty} = \lim_{k \rightarrow + \infty} F(c^k) \leq F(c^0).
\end{equation}
\noindent Again from (\ref{chA2-EQ04}), we obtain that the sequence of step sizes is square summable:
\[
\sum_{k=0}^{+\infty}\left\|c^{k+1}-c^k \right\|_2^2 \leq \frac{2}{\min_{j} k_{jj}} \left(F(c^0)-F_{\infty}\right) < +\infty.
\]
\noindent In particular, step-sizes are also uniformly bounded:
\begin{equation}\label{chA2-EQ06}
t_i^2(c^k) = \left\|c^{k+1}-c^k \right\|_2^2 \leq \frac{2}{\min_{j} k_{jj}}  \left(F(c^0)-F_{\infty}\right) < + \infty.
\end{equation}

\noindent Now, fix any coordinate $i$, and consider the sequence $c_i^k$. Let $h_{ij}$ denote the subsequence of indices in which the $i$-th component is picked by the essentially cyclic rule and observe that
\[
c_i^{h_{ij}} = S_i\left(\frac{k_i^T c^{h_{ij}-1}}{k_{ii}}-c_i^{h_{ij}-1}\right).
\]
\noindent Recalling the definition of $S_i$, and after some algebra, the last equation can be rewritten as
\[
c_i^{h_{ij}} \in - \partial f_i\left(k_i^T c^{h_{ij}-1}+ k_{ii} t_i\left(c^{h_{ij}-1}\right)\right).
\]

\noindent Since $\partial f_i (x)$ is a compact set for any $x \in \mathbb{R}$, it suffices to show that the argument of the subdifferential is bounded. For any $k$, let's decompose $c^k$ as
\[
c^k = u^k+v^k, \qquad \mathbf{K}u^k = 0, \qquad \langle u^{k}, v^k\rangle_2 = 0.
\]
\noindent Letting $\alpha_1 > 0$ denote the smallest non-null eigenvalue of the kernel matrix, we have
\[
\alpha_1 \|v^k\|_2^2 \leq v^{kT}\mathbf{K} v^k = c^{kT}\mathbf{K} c^k  \leq 2 F(c^k) \leq 2 F(c^0).
\]
\noindent By the triangular inequality, we have
\[
\left|k_i^T c^{k}+ k_{ii} t_i\left(c^{k}\right)\right|  \leq M \left|\frac{k_i^T c^{k}}{k_{ii}}+ t_i\left(c^{k}\right)\right|  \leq M \left( \left|\frac{k_i^T c^{k}}{k_{ii}}\right| + \left|t_i\left(c^{k}\right)\right|\right),
\]
\noindent where $M := \max_{j}|k_{jj}|$. The first term can be majorized as follows:
\[
\left|\frac{k_i^T c^{k}}{k_{ii}}\right| = \left|\frac{k_i^T v^{k}}{k_{ii}}\right| \leq \left\|\frac{k_i}{k_{ii}}\right\|_2\|v^k\|_2 \leq \left\|\frac{k_i}{k_{ii}}\right\|_2\sqrt{\frac{2 F(c^0) }{\alpha_1}} \leq \sqrt{\frac{ 2 \ell  F(c^0) }{\alpha_1}} < +\infty,
\]
\noindent while the term $\left|t_i\left(c^{k}\right)\right|$ is bounded in view of (\ref{chA2-EQ06}). It follows that $c_i^k$ is bounded independently of $i$, which implies that $c^k$ is bounded. In particular, the subsequence consisting of the macro-iterations is bounded as well.

By Theorem \ref{chA1-THM02}, there's at least one subsequence of the sequence of macro-iterations converging to a limit $c_{\infty}$ that satisfies (\ref{ch02-EQ08}), and thus minimizes $F$. By continuity of $F$, we have
\[
F(c_{\infty}) = \min_{c \in \mathbb{R}^{\ell}} F(c).
\]

\noindent Finally, in view of (\ref{chA2-EQ05}), we have $F_{\infty} = F(c_{\infty})$, which proves (\ref{ch02-EQ09}) and shows that any cluster point of $c^k$ is an optimal solution of Problem (\ref{ch02-EQ02}).
\end{proof}

\begin{proof}[Proof of Theorem \ref{ch02-THM05}]
Equation (\ref{ch02-EQ07}) can be rewritten as
\[
J_{\alpha}\left( \mathbf{K}_{\alpha} c\right) + c = 0.
\]
\noindent Now, let $f_{\alpha}$ denote the Moreau-Yosida regularization of $f$. From the properties of $f_{\alpha}$, we have
\[
\nabla f_{\alpha}(\mathbf{K}_{\alpha} c) +  \alpha c = 0.
\]
\noindent Multiplying both sides of the previous equation by $\alpha^{-1} \mathbf{K}_{\alpha}$, we obtain
\[
\alpha^{-1}\mathbf{K}_{\alpha} \nabla f_{\alpha}(\mathbf{K}_{\alpha} c) + \mathbf{K}_{\alpha} c = 0.
\]
\noindent Finally, the last equation can be rewritten as
\[
\nabla_{c} \left[\alpha^{-1}f_{\alpha}\left(\mathbf{K}_{\alpha} c\right)+\frac{c^T\mathbf{K}_{\alpha} c}{2} \right] = 0,
\]
\noindent so that the thesis follows.
\end{proof}

\end{document}